\documentclass[11pt, twocolumn]{article}
\usepackage[round]{natbib}
\usepackage{graphicx}
\usepackage[utf8]{inputenc}
\usepackage{color}
\usepackage{soul}
\usepackage{booktabs}
\usepackage{placeins}
\usepackage{enumitem}
\usepackage[utf8]{inputenc}
\usepackage{fancyhdr}
\usepackage{hyperref}

\usepackage{tikz}
\usetikzlibrary{arrows}

\usepackage{amsmath,amssymb}
\usepackage{float}

\usetikzlibrary{positioning}
\usepackage{enumitem}

\usepackage{longtable}
\usepackage{array}
\usepackage{geometry}
\usepackage[american]{babel}

\usepackage{mathtools} % amsmath with fixes and additions
\usepackage{booktabs} % commands to create good-looking tables
\usepackage{tikz} % nice language for creating drawings and diagrams
\usepackage{enumitem}
\usepackage{multirow}

\usepackage{color,soul}
\usepackage{comment}

\usepackage{amssymb,dsfont}
\usepackage{amsthm}

\newtheorem{thmprop}{Proposition}
\newtheorem{thmthm}{Theorem}

\newtheorem*{thmex*}{Example}
\newtheorem*{thmprop*}{Proposition}
\newtheorem*{thmthm*}{Theorem}

\newtheorem{thmcond}{Condition}

\newenvironment{proofsketch}{%
\proof}{\endproof}

\usepackage{makecell}

\newcommand{\suml}{\sum\limits}
\newcommand{\nwl}{\nonumber\\}
\newcommand{\bartheta}{\bar{\theta}}

\def\E{\mathbb{E}}

\def\cM{\mathcal{M}}

\def\cX{\mathcal{X}}
\def\cF{\mathcal{F}}

\def\na{{\tt NA}}

\def\ds1{\mathds{1}}

\def\bbR{\mathbb{R}}
\def\na{{\tt NA}}

\def\obs{{\lnot{}m}}
\def\Obs{{\lnot{}M}}
\def\miss{m}

\def\obsi{{\lnot{}m^{(i)}}}

\def\spsm{{\tt SPSM}}
\def\ridge{{\tt Ridge}}

\def\psm{{\tt PSM}}
\def\mlp{{\tt MLP}}
\def\lr{{\tt LR}}
\def\xgb{{\tt XGB}}
\def\lasso{{\tt LASSO}}

\title{Sharing Pattern Submodels for Prediction with Missing Values}
\author{Lena Stempfle, Ashkan Panahi, Fredrik Johansson}
\date{Department of Computer Science and Engineering (CSE), Chalmers University of Technology, Sweden\\ \href{mailto:stempfle@chalmers.se}{stempfle@chalmers.se} \href{mailto:ashkan.panahi@chalmers.se}
{ashkan.panahi@chalmers.se}  \href{mailto:fredrik.johansson@chalmers.se}
{fredrik.johansson@chalmers.se}}
  
\begin{document}
\maketitle

\begin{abstract}
Missing values are unavoidable in many applications of machine learning and present challenges both during training and at test time. When variables are missing in recurring patterns, fitting separate pattern submodels have been proposed as a solution. However, fitting models independently does not make efficient use of all available data. Conversely, fitting a single shared model to the full data set relies on imputation which often leads to biased results when missingness depends on unobserved factors. We propose an alternative approach, called sharing pattern submodels, which i) makes predictions that are robust to missing values at test time, ii) maintains or improves the predictive power of pattern submodels, and iii) has a short description, enabling improved interpretability. Parameter sharing is enforced through sparsity-inducing regularization which we prove leads to consistent estimation. Finally, we give conditions for when a sharing model is optimal, even when both missingness and the target outcome depend on unobserved variables. Classification and regression experiments on synthetic and real-world data sets demonstrate that our models achieve a favorable tradeoff between pattern specialization and information sharing.
\end{abstract}

\section{Introduction}\label{sec:intro}
Machine learning models are often used in settings where model inputs are partially missing either during training or at the time of prediction~\citep{rubin1976inference}. If not handled appropriately, missing values can lead to increased bias or to models that are inapplicable in deployment without imputing the values of unobserved variables~\citep{liu2020robust,morvan2020neumiss}. When missingness is dependent on unobserved factors that are related also to the prediction target, the fact that a variable is unmeasured can itself be predictive---so-called \emph{informative missingness}~\citep{rubin1976inference, marlin}. Often, imputation of missing values is insufficient, and it can be beneficial to let models make predictions based on both the partially observed data and on indicators for which variables are missing~\citep{jones1996indicator, Gronwold2012}.
As mentioned in~\citet{lemorvan20a_linear}, even the linear model---the simplest
of all regression models---has not yet been thoroughly investigated with missing values and still reveals unexpected challenges.

\begin{figure}[t]
 \centering
  \includegraphics[width=0.9\columnwidth]{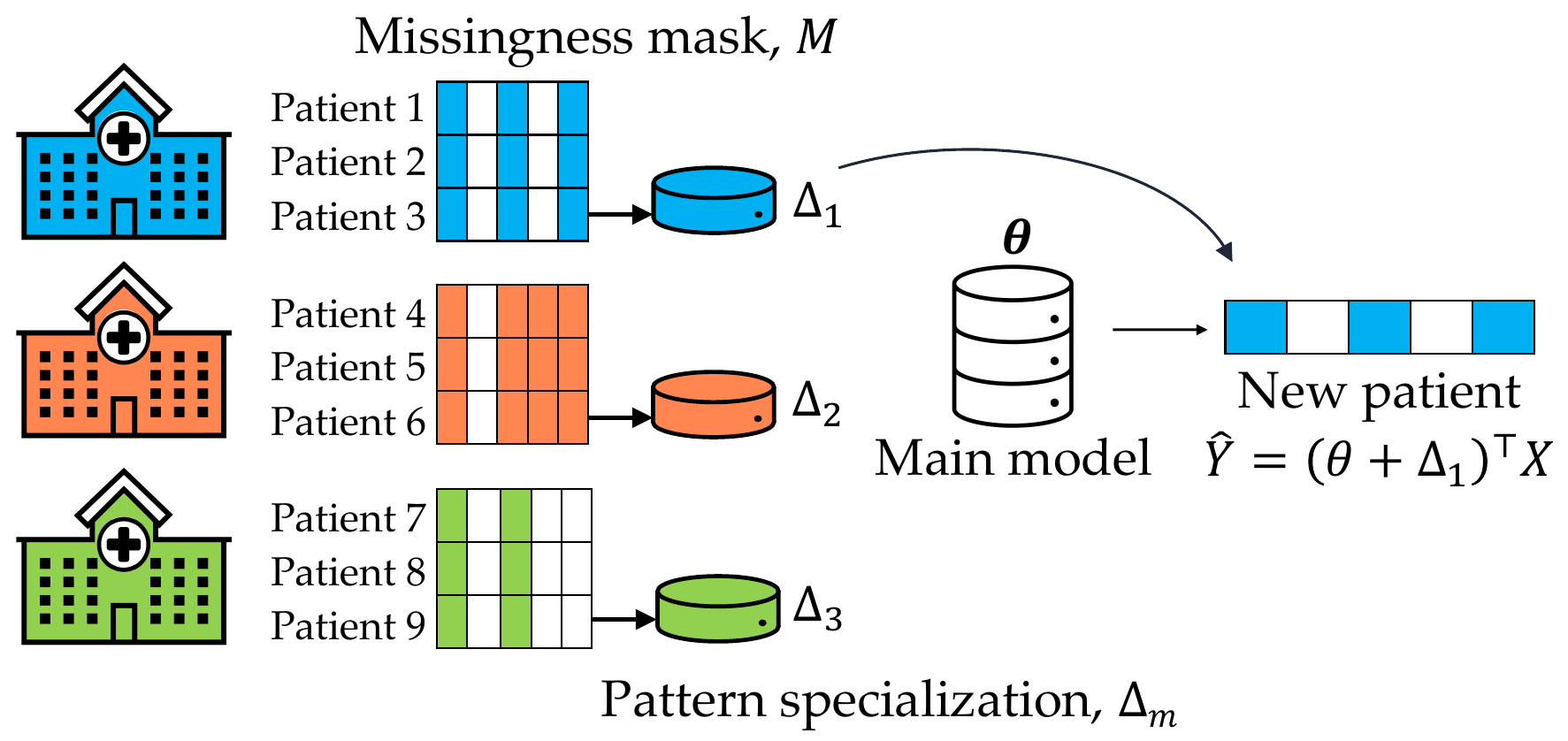}
  \caption{Coefficient sharing between a main model $\theta$ and pattern submodels for three clinics with different patterns in missing values. Without specialization,  $\Delta_m$, an average prediction shared by clinics with different patterns may not lead to an optimal solution for any of them. Conversely, fitting separate models for each clinic does not use all of the available data efficiently and leads to high variance.}\label{fig:example}
\end{figure}

\emph{Pattern missingness} emerges in data generating processes (DGPs) where there are structural reasons for which variables are measured---samples are grouped by recurring patterns of measured and missing variables~\citep{little1993pattern}. In Figure~\ref{fig:example}, we illustrate an example of this when observing patients from three different clinics, each systematically collecting slightly different measurements. Assume for simplicity that the pattern of missing values is unique to each clinic. In this way, a pattern-specific model is also site-specific. 

\emph{Pattern submodels} have been proposed for this setting, fitting a separate model to samples from each pattern~\citep{fletcher2020missing,marshall2002prospective}. This solution does not rely on imputation and can improve interpretability over black-box methods~\citep{rudin2019stop}, but can suffer from high variance, especially when the number of distinct patterns is large and the number of samples for a given pattern is small. Moreover, if the fitted models differ significantly between patterns, it may be hard to compare or sanity-check their predictions. Notably, pattern submodels disregard the fact that the prediction task is shared between each pattern.
However, in the context of Figure~\ref{fig:example}, using a shared model for all clinics may also be suboptimal if clinics take different measurements, or treat patients differently (high bias). 

We propose the \emph{sharing pattern submodel} (\spsm{}) in which submodels for different missingness patterns share coefficients while allowing limited specialization. This encourages efficient use of information across submodels leading to a beneficial tradeoff between predictive power and variance in the case where similar submodels are desired and sample sizes per pattern are small. Additionally, models with few and small differences between patterns are easier for domain experts to interpret.

We describe \spsm{} in Section~\ref{sec:SPSM}, and we prove that in linear-Gaussian systems, a model which shares coefficients between patterns may be optimal---even when the prediction target depends on missing variables and on the missingness pattern (Section 4). Finally, we find in an experimental evaluation on real-world and synthetic data that \spsm{} compares favorably to baseline classifiers and regression methods, paying particular attention to how \spsm{} boosts sample efficiency and model sparsity (Section~\ref{sec:exp}).

\section{Prediction with Test-Time Missingness}\label{sec:prob}
Let $X = [X_1, ..., X_d]^\top$ be a vector of $d$ random variables taking values in $\cX \subseteq \mathbb{R}^d$,  
and $M  = [M_1, ..., M_d]^\top$ be a random missingness mask in $\cM \subseteq \{0,1\}^d$  where
$M_j = 1$ indicates that variable $X_j$ is missing. Next, let $\tilde{X} \in
(\mathbb{R} \cup \{\na\})^d$ be the mixed observed-and-missing values of $X$ according to $M$ and define 
$
X_\Obs = [X_{j} : M_j=0]^\top \in \mathbb{R}^{d-\|M\|_1}
$
to be the vector of \emph{observed} covariates under $M$. The outcome of interest, $Y
\in \mathbb{R}$, may depend on all of $X$, observed or missing, as well as on $M$. Let $k = |\cM|$ denote the number of possible missingness patterns.\footnote{In practical scenarios, we expect $k$ to be much smaller than the worst-case number, $2^d$.} Further, assume that variables $X, M, Y$ are distributed according to a \emph{fixed, unknown} joint distribution $p$. The assumed (causal) dependencies of the variables used, coincide most closely with \emph{selection missingness}~\citep{little1993pattern} (Figure~\ref{fig:DAG} in the appendix).

Our goal is to predict $Y$ \emph{under missingness} $M$ in $X$ using functions $f : (\mathds{R} \cup
\{\na\})^d
\rightarrow \mathds{R}$. We aim to minimize risk with respect to the  squared loss on $p$,
\begin{equation}\label{eq:risk}
    \min_{f} R(f), \mbox{ where } R(f) \coloneqq \E_{\tilde{X},Y\sim p}[(f(\tilde{X}) - Y)^2]~.
\end{equation}
Under the assumption that $Y$ has centered, additive noise, 
\begin{equation}\label{eq:main_dgp}
Y = g(X, M) + \epsilon\;\mbox{ where }\; \E[\epsilon]=0,
\end{equation}
the Bayes-optimal predictor of $Y$ is  $f^* = \E[Y \mid X_\Obs, M]$.
In general, observed values $X_\Obs$ are insufficient for predicting $Y$; $f^*$ may depend directly on the mask $M$, \emph{even if $Y$ does not depend directly on $M$}~\citep{morvan2021s}. 

A common strategy to learn $f$ is to first impute the missing values in $\Tilde{X}$ and then fit a model on the observed-or-imputed covariates $X^I \in \mathbb{R}$---so-called \emph{impute-then-regress} estimation. Even though imputation is powerful, it is not always optimal under test-time missingness~\citep{morvan2021s} and often assumes that data is missing at random (MAR)~\citep{carpenter2012multiple, seaman2013meant}.
\subsection{Pattern Submodels}\label{subsec:PSM}
In cases where the number of distinct missingness patterns $k$ is small, it is possible to learn separate predictors $f_m$ for each  pattern. This idea has been called \emph{pattern submodels} (PSM)~\citep{fletcher2020missing,marshall2002prospective}, a set of models which aim to minimize the empirical risk under each missingness pattern. 
Let $D = \{(\tilde{x}^{(1)}, m^{(1)}, y^{(1)}), ..., (\tilde{x}^{(n)}, m^{(n)}, y^{(n)}) \}$ be a data set of $n$ samples, with partially observed features $\tilde{x}^{(i)}$, corresponding to missingness patterns $m^{(i)}$, drawn independently and identically distributed from $p$. PSM may be learned by minimizing the regularized empirical risk, 
\begin{equation}
\min_{\{f_m\} \in \cF^k} \;\; \frac{1}{n}
\sum_{i=1}^n L(f_{m^{(i)}}(\tilde{x}^{(i)}), y^{(i)}) + \sum_{m\in \cM} \mathcal{R}(f_m)
\label{eq:submodel}
\end{equation}
over a suitable class of models $\cF$ and regularization $\mathcal{R}$. \citet{fletcher2020missing} considered linear and logistic regression models, 
$f_m = \sigma(\theta_m^\top x)$ with $\sigma$ either the identity or logistic function and loss $L$ chosen to match. The objective in~(\ref{eq:submodel}) is separable in $m$ and can be solved independently for each pattern. However, this often leads to high variance in the small-sample regime since each pattern accounts for only a subset of the available samples. Without structural assumptions, the number of patterns $k$ grows exponentially with $d$ (see discussion in Section~\ref{sec:related}).

PSM allows for prediction under test-time missingness which adapts to the pattern $m$ without relying on imputation or assumptions on missingness mechanisms like MAR. However, the prediction target (and the Bayes-optimal model $f^*$) may have only a small dependence on the  pattern $m$; \emph{the optimal submodels for all $m$ may share significant structure}. Next, we propose estimators that exploit such structures to reduce variance and increase interpretability. 

%
% SHARING PATTERN SUBMODELS
%
\section{Sharing Pattern Submodels}\label{sec:SPSM}
We propose \textit{sharing pattern submodels} (\spsm{}), linear prediction models, specialized for patterns in variable missingness, which share information during learning. Sharing is accomplished by regularizing submodels towards a main model and solving the resulting coupled optimization problem. While linear models are limited in expressive power, they are often found to be useful approximations of nonlinear functions due to their superior interpretability.
%In general, we do not make assumptions on the functional form of $Y$ or the distribution of the noise $\epsilon$; linear models are often found to be useful approximations of nonlinear functions. In the case where the DGP is linear as well, we prove consistency of our approach (Section 4). 

\paragraph{Fitting SPSM} Let $\theta \in \mathbb{R}^d$ represent \emph{main model}  coefficients used in prediction under all missingness patterns, and define $\theta_\obs = [\theta_j : m_j = 0]^\top \in \mathbb{R}^{d_m}$ to be the subset of coefficients corresponding to variables observed under $m$. To emphasize, $\theta_\obs$ depends only on $m$ in selecting a subset of $\theta$---the coefficients are shared across patterns. Similarly, define $\Delta_\obs \in \mathbb{R}^{d_m}$ to be \emph{pattern-specific specialization} of these coefficients to $m$. In contrast to $\theta_\obs$, the values of $\Delta_\obs$ are unique to each pattern $m$. Note, a model $f_m$ depends only on the observed components of $X$. In regression tasks, we learn \textbf{sharing} pattern submodels on the form
\begin{equation}\label{eq:spsm_form}
f_m(x) \coloneqq (\theta_\obs + \Delta_\obs)^\top x_\obs,  \;\;\mbox{for all}\;\; m \in \cM
\end{equation}
by solving the following problem with $\lambda_m \geq 0$ and $\gamma \geq 0$,
\begin{align}\label{eq:spsm_ols}
    \underset{\theta, \{\Delta_\obs\}}{\text{minimize}}\;\;  & \frac{1}{n}\sum_{i=1}^n \big((\theta_\obsi + 
     \Delta_\obsi)^\top x^{(i)}_\obsi - y^{(i)} \big)^2  \nonumber \\ & + \frac{\gamma}{n} \|\theta\| +  \sum_{m \in \mathcal{M}}
     \frac{\lambda_m}{n_m} \|\Delta_\obs\|_1~.
\end{align}
where $n_m$ is the number of samples of pattern $m$. $\lambda_m>0$ and $\gamma>0$ are regularization parameters. Intercepts (pattern-specific and shared) are left out for brevity.  The optimization problem is convex, and we find optimal values for $\theta$ and $\Delta_m$ using L-BFGS-B~\citep{byrd1995limited} in experiments. In classification tasks, the square loss is replaced by the logistic loss. In either case, we call the solution to \eqref{eq:spsm_ols} \spsm{}.

For the penalty $\|\theta\|$, we use either the $\ell_1$ or $\ell_2$ norm to tradeoff bias and variance in the main model. A high value for $\lambda_m$ regularizes the specialization of model coefficients to missingness pattern $m$ such that high $\lambda_m$ encourages smaller $\|\Delta_m\|_1$ and greater coefficient sharing. In experiments, we let $\lambda_m$ take the same value $\lambda$ for all patterns. $\ell_1$-regularization is used for $\Delta$ as we aim for a sparse solution where the majority of specialization coefficients are zero. 

\paragraph{Consistency} For fixed $\lambda, \gamma$, sums of the minimizers of \eqref{eq:spsm_ols}, $\theta_\obs^* + \Delta_\obs^*$, converge to the best linear approximations of the Bayes-optimal predictors $f_m^*$ for each pattern $m$ in the large-sample limit.
We state this formally and sketch a proof in Appendix~\ref{app:consistency} using standard arguments. This result is agnostic to parameter sharing; $\Delta^*$ may not be sparse. In Section~\ref{sec:analysis}, we prove that, in the linear-Gaussian setting, our method also recovers the sparsity of the true process. In the large-sample limit, this may not be beneficial for variance reduction, but sparsity contributes to interpretability.

\paragraph{Why is SPSM Interpretable?}
Comparing pattern specializations allows domain experts to reason about how similar submodels are, and how they are affected by missing values. We argue that a set of submodels is more interpretable if specializations contain fewer non-zero coefficients, $\Delta_\obs$ is sparse. Sparsity is a generally useful measure of interpretablity~\citep{rudin2019stop}, since it results in only a subset of the input features affecting predictions, reducing the effective complexity of the model~\citep{miller1956magical, cowan2010magical}.
%
% ANALYSIS
%
\section{Optimality of Sharing Models}\label{sec:analysis}%
In this section, we give conditions under which an optimal pattern submodel has sparse specializations (shares parameters between patterns) and when \spsm{} converges to such a model in the large-sample limit. 
We analyze DGPs where the outcome $Y$ depends linearly on \emph{all} components of $X$ (\emph{models} have access only the observed subset of these) and on the pattern $M$, but not on interactions between $X$ and $M$,
\begin{equation}\label{eq:main_dgp_lin}
Y = \theta^\top X + \alpha_M + \epsilon~, \mbox{ with } \epsilon \sim \mathcal{N}(0, \sigma_Y^2).
\end{equation}
Here, $\alpha_M$ is a pattern-specific intercept. Without $\alpha_M$, this is a setting often targeted by imputation methods, since the outcome is a parametric function of the full $X$. However, we know that $X$ will be partially missing also at test time, and $M$ is allowed to have arbitrary dependence on $X$. In this case, imputation need not be necessary or sufficient. 

Next, we study this setting with Gaussian $X$, where we can precisely characterize optimal models and their sparsity.

\subsection{Sparsity in Linear-Gaussian DGPs}\label{sec:linear_gaussian}
Recall that $X_{\obs}$ and $\theta_{\obs}$ denote covariates and coefficients restricted to \emph{observed} variables under pattern $m$, and define $X_\miss$ and $\theta_\miss$ analogously for missing variables. 
For outcomes which obey \eqref{eq:main_dgp_lin}, the Bayes-optimal model under $m$ is 
\begin{align}\label{eq:pre_derivation}
        & \mathbb{E}[Y \mid X_\obs, M=m]={\theta_\obs}^\top X_\obs + \xi_\miss
\end{align}
where 
$
\xi_\miss =
{\theta_{m}}^\top \mathbb{E}_{X_{m}}[ X_{m} \mid X_\obs] + \alpha_m 
$ 
is the bias of the na\"{i}ve prediction made using the coefficients $\theta_\obs$ of the true system but restricted to observed variables. Ignoring $\xi_\miss$ coincides with performing prediction following  zero-imputation and is biased in general. $\xi_\miss$ thus captures the specialization required for pattern submodels to be unbiased. For closer analysis, we study the following setting.
\begin{thmcond}[Linear-Gaussian DGP]\label{cond:lineargauss}%
Covariates $X = [X_1, ..., X_d]^\top$ are Gaussian, $X \sim \mathcal{N}(\boldsymbol{\mu}, \Sigma)$ with mean $\boldsymbol{\mu}$ and covariance matrix $\Sigma$. The outcome $Y$ is linear-Gaussian as in \eqref{eq:main_dgp_lin} with parameters $(\theta, \{\alpha_m\}, \sigma_Y)$. $M$ is arbitrary.%
\end{thmcond}
In line with Condition~\ref{cond:lineargauss}, let $\Sigma_{\obs,\miss}$ be the submatrix of $\Sigma$ restricted to the rows corresponding to \emph{observed} variables under $m$ and columns corresponding to variables \emph{missing} under $m$. Define $\Sigma_{\obs,\obs}$ and $\Sigma_{\miss,\obs}$ analogously. Throughout, we assume that $\Sigma$ is invertible so that the distribution is non-degenerate. In practice, the non-degenerate case can be handled through ridge regularization. 
\begin{thmprop}\label{prop:main}
Suppose covariates $X$ and outcome $Y$ obey Condition~\ref{cond:lineargauss} (are linear-Gaussian). Then, the Bayes-optimal predictor for an arbitrary missingness mask $m \in \cM$, is
\begin{equation*}
    f^*_m = \mathbb{E}[Y \mid X_\obs, m] = (\theta_\obs + \Delta_\obs)^\top X_\obs + C_m
\end{equation*}
where $C_m \in \mathbb{R}$ is constant with respect to $X_\obs$ and
\begin{equation*}
    \Delta_\obs = ({\Sigma^{-1}_{\obs, \obs}}) \Sigma_{\obs, \miss}  \theta_\miss~.%
\end{equation*}%
\end{thmprop}
Proposition~\ref{prop:main} states that, for a linear-Gaussian system, the Bayes-optimal model under missingness pattern $m$ has the same form as \spsm{} with pattern-specific intercept, combining coefficients of a main model $\theta$ and specializations $\Delta_\obs$. The result is proven in Appendix~\ref{app:thm}.

In nonlinear DGPs, the optimal correction term $\Delta_\obs$ may not be constant with respect to $X_\obs$. The NeuMiss model by ~\citet{morvan2020neumiss} learns such corrections as functions of the input and missingness mask using deep neural networks. However, this method lacks the interpretability of sparse linear models sought here.
Even in this more general case, \spsm{} may achieve a good bias-variance tradeoff. Indeed, we find on real-world data, which may not be linear, that \spsm{} is often preferable to strong nonlinear baselines. 

\subsubsection{When is Sparsity Optimal?}
Like other sparsity-inducing regularized estimators, such as  \lasso{}~\citep{tibshirani1996regression}, \spsm{}
reduces variance by shrinking some model parameters to zero. Under appropriate conditions, when the training set grows large, we expect the learned sparsity to correspond to properties inherent to the DGP. For \lasso{}, this means recovering zeros in the coefficient vector of the outcome. For \spsm{}, objective~\eqref{eq:spsm_ols} is used to learn submodels on the form $(\theta_\obs + \Delta_\obs)^\top X_\obs$ where $\theta$ is shared between patterns and $\Delta_\obs$ is sparse. It is natural to ask: When can we expect the ``true'' or an ``optimal'' $\Delta_\obs$ to be sparse and, if it is, when can we recover this sparsity with \spsm{}? Surprisingly, as we will see, the optimal specialization $\Delta_\obs$ may be sparse even if $Y$ depends on \emph{all} covariates in $X$.

Assume that Condition~\ref{cond:lineargauss} (Linear-Gaussian DGP) holds with system parameters $(\mu, \Sigma, \theta, \{\alpha_m\}, \sigma_Y)$. We can characterize sparsity in the Bayes-optimal model $(\theta, \{\Delta_\obs\})$, see Proposition~\ref{prop:main}, by the interactivity of covariates. We say that variables $X_j$ and $X_{j'}$ are non-interactive if they are statistically independent given all other covariates. As is well-known, for Gaussian $X$, $X_j$ and $X_{j'}$ are non-interactive if $S_{j,j'} = 0$, where $S = \Sigma^{-1}$ is the precision matrix.
\begin{thmprop}[Sparsity in optimal model]\label{prop:sparsity}
Suppose that a covariate $j\in [d]$ is observed under pattern $m$, i.e., $m_j = 0$, and assume that $X_j$ is non-interactive with every covariate $X_{j'}$ that is missing under $m$. Then $(\Delta_\obs)_j = 0$. %
\end{thmprop}%
Proposition~\ref{prop:sparsity} states that the sparsity in $\Delta$ is partially determined by the covariance pattern of observed and unobserved covariates. For example, specialization is \emph{not needed} for a variable $j$ under pattern $m$ if it is uncorrelated with all missing variables under $m$. 
Conversely, specialization, i.e., $(\Delta_\obs)_j \neq 0$, \emph{is needed} for features $j$ that are predictive ($\theta_j \neq 0$) and redundant (replicated well by unobserved features which are also predictive). This is because in the main model, redundant variables may share the predictive burden, but when they are partitioned by missingness, they have to carry it alone. This shows that prediction with a single model and zero-imputation is sub-optimal in general.

\subsubsection{Consistency of SPSM}
In the large-sample limit, under Condition~\ref{cond:lineargauss}, we can prove that \spsm{} recovers maximally sparse optimal model parameters. If the true system parameters are also sparse, \spsm{} learns these.
\begin{thmthm}\label{thm:consistency}
Suppose that Condition~\ref{cond:lineargauss} holds with parameters $(\theta, \{\Delta_\obs\})$ as in Proposition~\ref{prop:main}, such that, for each covariate $j$, the number of patterns $m$ for which $m_j = 0$ and $(\Delta_\obs)_j = 0$ is strictly larger than the number of patterns $m'$ for which $m'_j=0$ and $(\Delta_{\obs'})_j \neq 0$. 
Then, with $\gamma=0$ and fixed $\lambda >0$, the true parameters $(\theta, \{\Delta_\obs\})$ are  the unique  solution to \eqref{eq:spsm_ols} in the large-sample limit,  $n \rightarrow \infty$.

\end{thmthm}
\begin{proofsketch}
We provide a full proof in Appendix~\ref{app:consistency_lin}. The main steps involve showing that the SPSM objective \eqref{eq:spsm_ols} is asymptotically dominated by the risk term, and the \emph{sums} of its minimizers ($\theta_\obs^*$ +
$\Delta^*_{\obs}$) coincide with optimal regression coefficients  ($\hat{\theta}_{\obs}$) fit independently for each missingness pattern $m$. For any $\lambda > 0$, regularization steers the solution towards one which is maximally sparse in $\Delta^*_\obs$. 
\end{proofsketch}
%
% Relationship
%
\subsection{Relationship to Other Methods}
For particular extreme values of the regularization parameters $\gamma, \lambda_m$, \spsm{} coincides with other methods (Table~\ref{tab:related_methods}). First, the full-sharing model ($\lambda_m \rightarrow \infty, \gamma<\infty$) coincides with fitting a single model to all samples after zero-imputation. To see this, set $\Delta_\obs = 0$ for all $m$ and note 
$$
{\theta_\obsi}^\top x_\obsi^{(i)} = \theta^\top I_0(\tilde{x}^{(i)})
$$
where $I_0(\tilde{x})$ replaces missing values in $\tilde{x}$ with 0. In this setting, submodel coefficients cannot adapt to $m$. In the implementation, we allow the fitting of pattern-specific intercepts which are not regularized by $\lambda_m$.
\begin{table}
    \centering
    \begin{tabular}{r|c c}
        & $\gamma < \infty$ & $\gamma \rightarrow \infty$\\
        \hline
        $\lambda_m \rightarrow \infty$ &  Zero imputation & Constant\\
        $0 <\lambda_m < \infty$  & Sharing model & Pattern submodel\\
        $\lambda_m = 0$  & No sharing & Pattern submodel \\
    \end{tabular}
    \caption{Extreme cases and equivalences of \spsm{}, provided that no pattern-specific intercept is used. }\label{tab:related_methods}
\end{table}
Second, ($\lambda_m<\infty, \gamma \rightarrow \infty$) corresponds with the standard \psm{} without parameter sharing~\citep{fletcher2020missing} or the ExpandedLR method of~\citep{lemorvan20a_linear}. The precise nature of this equivalence depends on the choice of regularization.\footnote{\citet{fletcher2020missing} adopted a two-stage estimation procedure, the relaxed LASSO~\citep{meinshausen2007relaxed}.} In this setting, each submodel $\hat{f}_m$ is fit completely independently of every other. Finally, an \spsm{} model with optimal parameters $(\theta, \{\Delta_\obs\})$, in the linear-Gaussian case, implicitly makes a perfect single linear imputation, $$\mathbb{E}[X_\miss \mid X_\obs] = X_\obs{\Sigma^{-1}_{\obs, \obs}} \Sigma_{\obs, \miss},$$ and applies the main model’s parameters $\theta_\miss$ to the imputed values.  If many samples are available, it may be feasible to learn the imputation  directly. However, if the variables in $X_\obs$ and $X_m$ are never observed together, imputation is no longer possible. In contrast, \spsm{} could still learn an optimal submodel for each pattern, given enough samples. 

%
% EXPERIMENTS
%
\section{Experiments}\label{sec:exp}
We evaluate the proposed \spsm{} model\footnote{Code to reproduce experiments and the appendix are available at \url{https://github.com/Healthy-AI/spsm}.} on simulated and on real-world data, aiming to answer two main questions: How does the accuracy of \spsm{} compare to baseline models, including impute-then-regress, for small and larger samples; How does sparsity in pattern specializations $\Delta$ affect performance and interpretation?

\subsubsection*{Experimental Setup}
In the \spsm{} algorithm, before one-hot-encoding of categorical features, all missingness patterns in the training set are identified. At test time, patterns that did not occur during training, variables are removed until the closest training pattern is recovered.
Both linear and logistic variants of \spsm{} were trained using the L-BFGS-B solver provided as part of the SciPy Python package~\cite {virtanen2020scipy}. Our implementation supports both $\ell_1$ and $\ell_2$-regularization of the main model parameters $\theta$ and $\ell_1$-regularization of pattern-specific deviations $\Delta$. This includes both the no-sharing pattern submodel ($\lambda_m<\infty, \gamma \rightarrow \infty$) and full-sharing model ($\lambda_m \rightarrow \infty, \gamma<\infty$) as special cases. In the experiments, $\gamma$ can take values within $[0, 0.1, 1, 5, 10, 100]$, and we used a shared $\lambda_m = \lambda \in [1, 5, 10, 100, 1000, 1e^8]$ for all patterns. 
Intercepts were added for both the main model and for each pattern without regularization. We do not require patterns to have a minimum sample size but support this functionality (appendix Table~\ref{tab:appendix_sample_sizes}). For missingness patterns at test time that did not occur in the training data, variables were removed until the closest training pattern was recovered. 

We compare linear and logistic regression models to the following baseline methods: Imputation + Ridge / logistic regression (\ridge/\lr), Imputation + Multilayer perceptron (\mlp) with a single hidden layer, and XGBoost (\xgb), where missing values are supported by default~\citep{chen2019package}. Last, we compare the Pattern Submodel (\psm)~\citep{fletcher2020missing}. Note, our implementation of \psm{} is based on a special case of our \spsm{} implementation where regularization is applied over all patterns and not in each pattern separately. Hyperparameters are based on the validation set. For imputation, we use zero ($I_0$), mean ($I_\mu$) or iterative imputation ($I_{it}$) from SciKit-Learn~\citep{scikit-learn,buren}. \xgb{}'s handling of missing values is denoted $I_n$.
Details about method implementations, hyperparameters and evaluation metrics are given in Appendix~\ref{app:B}.

\subsection{Simulated Data}
We use simulated data to illustrate the behavior of sharing pattern submodels and baselines in relation to Proposition~\ref{prop:main}, focusing on bias and variance. 
We sample $d$ input features $X$ from a multivariate Gaussian $\mathcal{N}(0, \Sigma)$ with covariance matrix $\Sigma$ specified by a cluster structure; the features are partitioned into $k$ clusters of equal size. The covariance is defined as $\Sigma_{ii} = 1$, $\Sigma_{i\neq j} = 0$ if $i,j$ are in different clusters, and $\Sigma_{i\neq j} = c$ if $i,j$ are in the same cluster, where $c$ is chosen as large as possible so that $\Sigma$ remains positive semidefinite. 

Each cluster $c \in \{1, ..., k\}$ is represented in the outcome function $Y = \theta^\top X + \epsilon$ by a single feature $i(c)$, such that $\beta_{i(c)} \sim \mathcal{N}(0,1)$ and $\theta_j = 0$ for other features. We let $\epsilon \sim \mathcal{N}(0,1)$, independently for each sample. 
We consider three missingness settings: In Setting A, each variable in cluster $c$ is missing if $X_{i(c)} > -0.5$. In Setting B, each variable in cluster $c$---except one chosen uniformly at random---is missing if $X_{i(c)} > -0.5$. Both settings satisfy the conditions of Proposition~\ref{prop:main} but are designed to violate MAR by letting the outcome variable depend directly on missing values which may not be recovered from observed ones. In Setting C, we follow missing-completely-at-random (MCAR), where variables are missing independently with probability 0.2. We generate samples with $d=20$ and $k=5$. 

In Figure~\ref{fig:synth_exp_A}, we show the test set coefficient of determination ($R^2$) for Setting A. Note, that the methods which use imputation (imputation method selected based on validation error at each data set size) perform well initially but plateau quickly, indicating relatively high bias. \spsm{} and \psm{} both achieve a higher $R^2$ for the full sample. \spsm{} performs better than \psm{} for small samples indicating lower variance. The \spsm{} model includes 42 non-zero pattern-specific coefficients when the training set size is 0.2 and 68 with the fraction is 0.8.
Results for Setting B and C are presented in Appendix~\ref{app:simu_resul}. Even in the MCAR setting C, \psm{} performs considerably worse than alternatives due to excessive variance from fitting independent pattern-specific models.
\begin{figure}[t]
    \centering
    \includegraphics[width=.8\columnwidth]{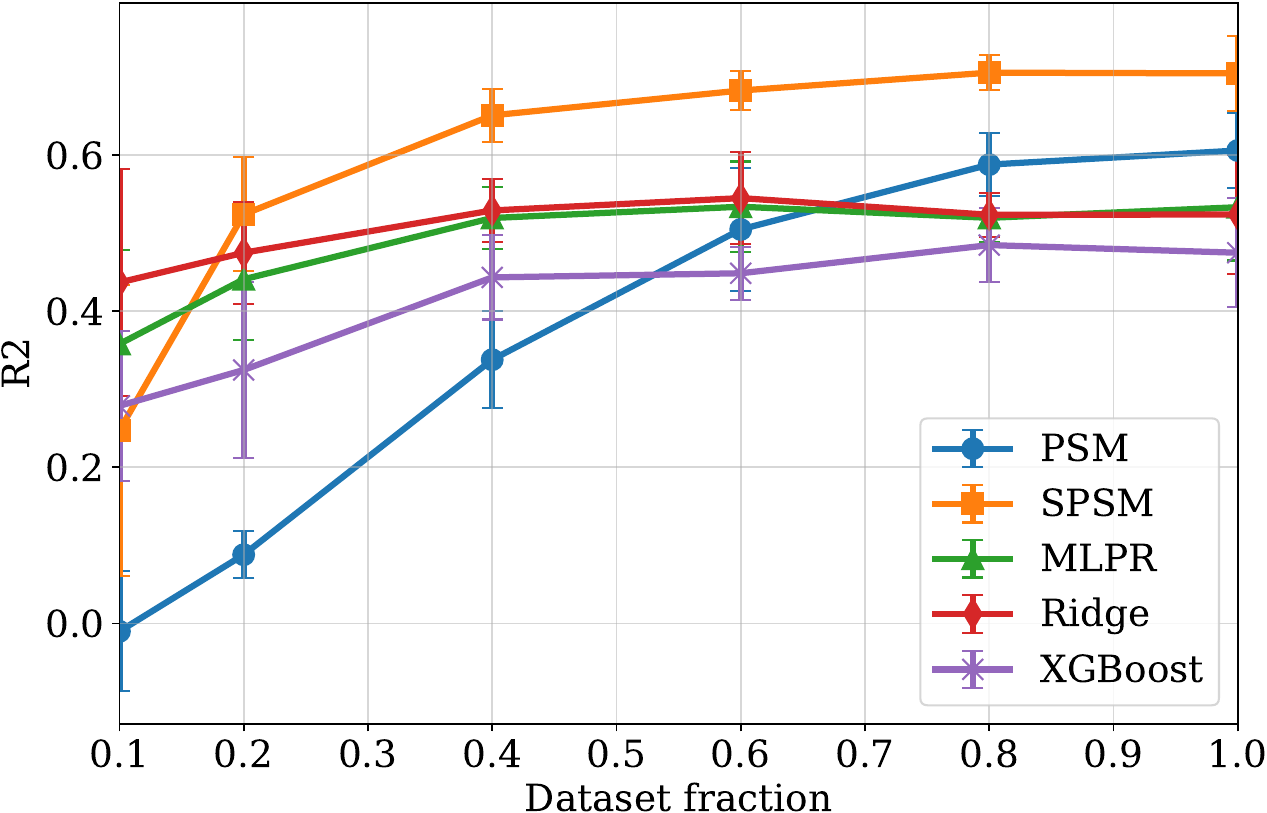}
    \caption{Performance on simulated data Setting A (higher is better). Error bars show standard deviation over 5  random data splits. The full data set has $n=2000$ samples.}
    \label{fig:synth_exp_A}
\end{figure}

\subsection{Real-World Tasks}\label{sec:datasets}
We describe two health care data sets used for classification and regression. More information on the non-health related HOUSING~\cite{de2011ames} data is shown in Appendix~\ref{appendix:housing}.  

\paragraph{ADNI}
The data is obtained from the publicly available Alzheimer's Disease Neuroimaging Initiative (ADNI) database.\footnote{\url{http://adni.loni.usc.edu}} ADNI collects clinical data, neuroimaging and genetic data~\citep{ADNI}. In the classification task, we predict if a patient's diagnosis will change 2 years after baseline diagnosis. The regression task aims to predict the outcome of the ADAS13 (Alzheimer's Disease Assessment Scale)~\citep{mofrad2021cognitive} cognitive test at a 2-year follow-up based on available data at baseline.

\paragraph{SUPPORT}
We use data from the Study to Understand Prognoses and Preferences for Outcomes and Risks of Treatments (SUPPORT) ~\citep{Knausstudy}, which aims to model survival over a 180-day period in seriously ill hospitalized adults using the Physiology Score (SPS). Following~\citet{fletcher2020missing}, in the regression task we predict the SPS while for the classification task, we predict if a patient's SPS is above the median; the label rate is 50/50 by definition. We mimic their MNAR setting by adding 25 units to the SPS values of subjects missing the covariate ''partial pressure
of oxygen in the arterial blood''.

\subsection{Results}\label{sec:results}
We report the results on health care data in Table~\ref{tab:data_results}. For regression tasks, we provide the number of non-zero coefficients used by the linear models. In addition, we study prediction performance as a function of data set size in Figure~\ref{fig:ADNI_fraction_reg} and in the appendix Figure~\ref{fig:SUPPORT_fraction_reg}. The statistical uncertainty of the average error is measured with its square root, which is a standard deviation and expressed by 95\% confidence intervals over the test set. Results of HOUSING data are presented in Appendix~\ref{appendix:housing}.

\begin{table}[t!]
    \centering
    \begin{small}
    \begin{tabular}{clll}
      \multicolumn{2}{l}{\bfseries Regression} & $\boldsymbol{R}^2$ & \bfseries \# Coefficients\\
        \multicolumn{3}{l}{ADNI} \\
        \hline
        & \ridge{}, $I_{\mu}$ & 0.66 (0.59, 0.73) & 37 + 0\\
        & \xgb{}, $I_{\mu}$  &  0.41 (0.31, 0.50)& --- \\
        & \mlp{}, $I_{0}$ &    0.62 (0.55, 0.69)& --- \\
        & \psm{} &   0.51 (0.43, 0.60)& 0 + 430 \\
        & \spsm{}  &  0.66 (0.59, 0.73) & 37 + 21 \\
        \hline
        \multicolumn{3}{l}{SUPPORT} \\
        \hline
        & \ridge{},  $I_0$ & 0.38 (0.35, 0.42)& 11 + 0\\
        &\xgb{}, $I_n$  &  0.30 (0.27, 0.34) & ---\\
        & \mlp{},  $I_\mu$ & 0.56 (0.53, 0.59) &  --- \\
        & \psm{} & 0.52 (0.49, 0.56) & 0 + 188\\
        & \spsm{} & 0.53 (0.50, 0.56)& 11 + 91\\
        \hline
        \hline
      \multicolumn{2}{l}{\bfseries Classification}  & \bfseries AUC & \bfseries Accuracy \\
        \multicolumn{3}{l}{ADNI} \\
        \hline
         & \lr{}, $I_0$ & 0.85 (0.80, 0.90) & 0.85 (0.74, 0.94)  \\
        & \xgb{}, $I_n$ & 0.80 (0.74, 0.86) & 0.84 (0.73, 0.94) \\
        & \mlp{}, $I_0$ & 0.86 (0.78, 0.89) & 0.84 (0.73, 0.94)\\
        & \psm{} & 0.81 (0.75, 0.87) & 0.84 (0.74, 0.95)\\
        & \spsm{}  & 0.86 (0.81, 0.90) & 0.85 (0.75, 0.96) \\
        \hline
        \multicolumn{3}{l}{SUPPORT} \\
        \hline
        & \lr{}, $I_0$  &  0.83 (0.81, 0.85) &  0.77 (0.74, 0.79) \\
        & \xgb{}, $I_0$ & 0.85 (0.83, 0.87)  & 0.78 (0.75, 0.81) \\
        & \mlp{}, $I_0$ & 0.86 (0.85, 0.88) & 0.79 (0.76, 0.81)\\
        & \psm{} & 0.84 (0.83, 0.86) & 0.78 (0.75, 0.81) \\
        & \spsm{}  & 0.85 (0.83, 0.86)  & 0.78 (0.75, 0.80)\\
    \end{tabular}
    \end{small}
    \caption{Results for ADNI and SUPPORT tasks along with the respective imputation method (see setup). We also report the number of non-zero coefficients in shared ($k$) and pattern-specific models ($l$) as $k + l$.}\label{tab:data_results}
\end{table}

%
%\paragraph{Regression}
%ADNI 
For ADNI regression, \spsm{} and \ridge{} are the best performing models with $R^2$ of 0.66 showing the same confidence in the prediction. Validation performance resulted in selecting $\gamma = 10.0, \lambda=50$ for \spsm{}. With an $R^2$ score of 0.51, \psm{} seems not able to benefit from pattern-specificity in ADNI. In contrast, \spsm{} makes use of coefficient sharing which results in a significantly smaller number of coefficients compared to \psm{}.
%SUPPORT
For SUPPORT regression, \psm{} achieves almost the same result as \spsm{} ($R^2$ of 0.52--0.53) with partly overlapping confidence intervals for the predictions. Although, the number of coefficients used in \spsm{} is smaller than in \psm{} due to the coefficient sharing between submodels. The best regularization parameter values for \spsm{} were $\gamma=0.1,  \lambda = 5.0$ which is lower than for ADNI, consistent with the larger data set size. The best performing model is \mlp{} ($R^2$ of 0.56) for SUPPORT regression. However, the black-box nature of \mlp{} is not conducive to reasoning about the influence of the missingness pattern. Mean and zero imputation have the best validation performance for \ridge{}, \xgb{} and \mlp{}.
In summary, \spsm{} is consistently among the best-performing models in both data sets, with fairly tight confidence intervals.
%ADNI
In ADNI classification, \spsm{}, \mlp{} and \lr{}  achieve the highest prediction accuracy (0.84--0.85) and Area Under the ROC Curve (AUC) (0.85--0.86). All methods perform similarly well on ADNI. \spsm{} selected $\gamma=0$ and $\lambda=1.0$ which indicates moderate coefficient sharing. 
%
%SUPPORT
For SUPPORT data, all models perform almost at the same level.
\xgb{} and \mlp{} perform slightly better than \spsm{} ($\gamma$ = 0.1, $\lambda=10.0$) and \psm{}. 
Across ADNI and SUPPORT \lr, \xgb{} and \mlp{}  predominantly use zero imputation. In all tasks, \spsm{} performs comparably or favorably to all other methods. The tight confidence intervals for classification in both data sets indicate high certainty in the result averages.
%
%HOUSING
\paragraph{Non-Healthcare Data and Coefficient Specialization}
In contrast to the previous data sets, where sharing coefficients is beneficial, we see for the HOUSING data, a large advantage from nonlinear estimation: the tree-based approach \xgb{} (Table~\ref{tab:appendix_housing}). It shows an $R^2$ of 0.76 and outperforms the other baseline methods for the regression task confirming the non-linearity of that data set. We also do not see the same positive effect in specializing (\psm, \spsm{} not better than \ridge{} with imputation). None of the missing value indicators show a significant feature importance level in \xgb{} which might indicate that pattern specialization is not necessary. For results on the HOUSING data, see Appendix~\ref{appendix:housing}.

\paragraph{Performance with Varying Training Set Size}
Figure~\ref{fig:ADNI_fraction_reg} shows the test $R^2$ for linear models trained on different fractions of ADNI data. Each set was subsampled into fractions $0.2, 0.4, 0.6, 0.8, 1.0$ of the full data set. Especially for small fractions, \spsm{} benefits from coefficient sharing and lower variance data compared to \psm{}. \ridge{} with mean imputation performs comparably. A similar figure for the SUPPORT is presented in appendix Figure~\ref{fig:SUPPORT_fraction_reg}. \spsm{} and \psm{} perform equally well across the fractions, whereas \ridge{} shows high error compared to both pattern submodels. 

\begin{figure}[t]
    \centering
    \includegraphics[width=.8\columnwidth]{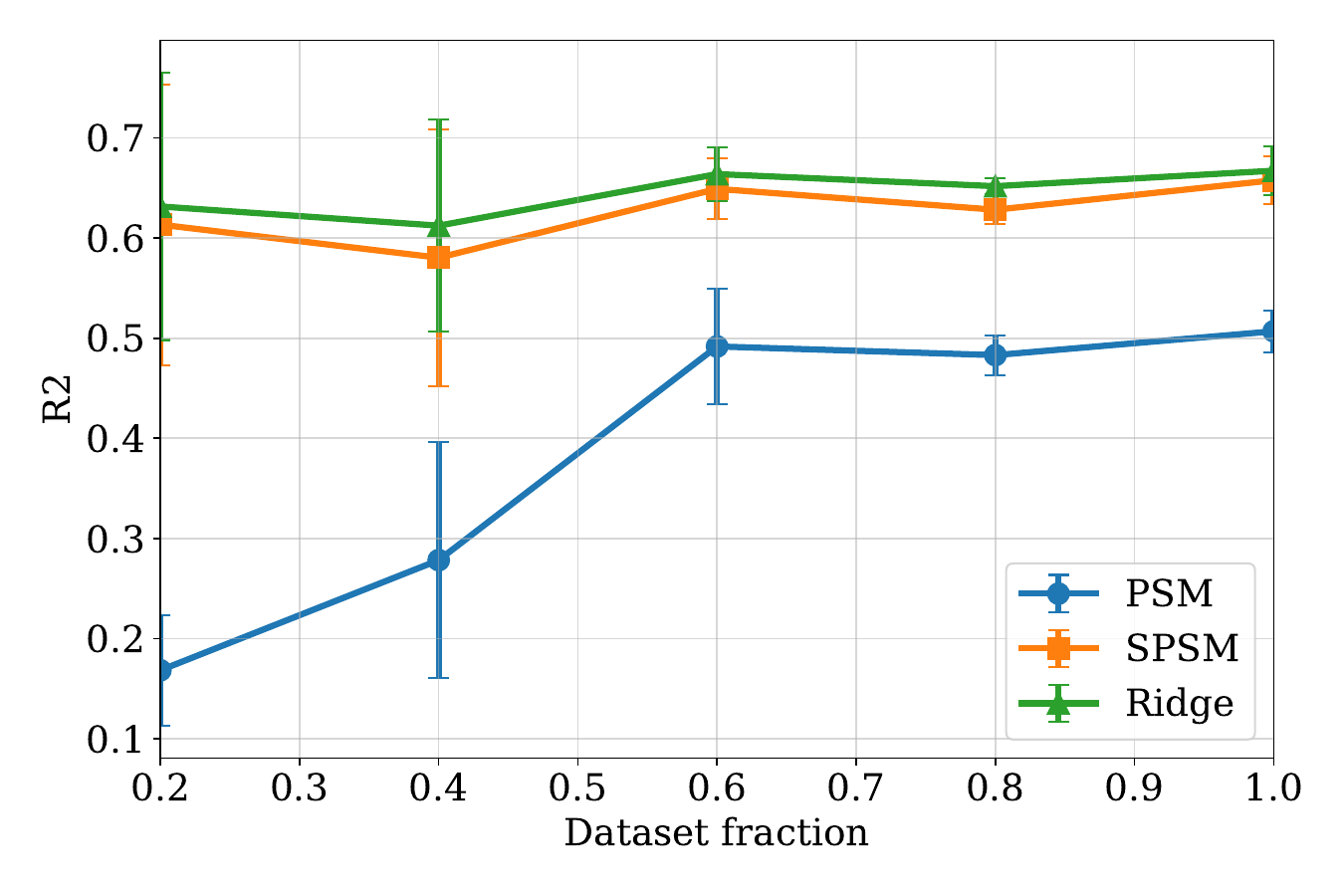}
    \caption{Performance on ADNI for the regression task. Error bars indicate standard deviation over 5 random subsamples of the data. Equal performance for \spsm{} and \ridge{} and subpar performance for \psm{} indicates that for ADNI regression, pattern specialization is mostly irrelevant.}
    \label{fig:ADNI_fraction_reg}
\end{figure}

\subsubsection*{Pattern Specialization in \spsm{}}
We inspect pattern specializations $\Delta$ for \spsm{} in the ADNI regression task with respect to interpretablity. In Table~\ref{tab:Coefficient_change}, we present the main model $\theta$ and pattern-specific coefficients $\Delta_4$ for pattern $4$. Table~\ref{tab:Full_table_all} in the appendix shows all patterns $m$ with $\Delta_\obs \neq 0$. For pattern 4, measurements of the amyloid-$\beta$ (ABETA) peptide and the proteins TAU and PTAU are missing in the baseline diagnostics. The absence of these three features affects pattern specialization: For an imaging test FDG-PET (fluorodeoxyglucose), the magnitude of its coefficient is increased, placing heavier weight on the feature in prediction. Similarly, the coefficients for Fusiform (brain volume), and ICV (intracranial volume) increase in magnitude and predictive significance when ABETA, TAU, and PTAU are absent. 
In contrast, for the feature AGE, the resulting coefficient of -0.019 (compared to 0.121 in the main model) means that the predictive influence of this feature decreases under pattern 4.
As Table~\ref{tab:Coefficient_change} shows, \spsm{} applied to tabular data allows for short descriptions of pattern specialization, which helps construct a simple and meaningful model.
%Sparsity in delta
We enforce sparsity in $\Delta$ to limit the number of differences between submodels, and present all features $j$ with specialized coefficients $\Delta_\obs(j) \neq 0$, five in the example case. In this way, the set of submodels is more interpretable and the user, e.g., a  medical staff member can be supported in decision-making. For a more detailed analysis on interpretability properties of \spsm{}, see Appendix~\ref{app:interpre}. 

\paragraph{Tradeoff between Interpretability and Accuracy}
The interpretability-accuracy tradeoff is especially crucial for practical use of \spsm{}. The empirical results do not show any significant evidence that our proposed sparsity regularization hurts prediction accuracy (Table~\ref{tab:data_results}, Figure~\ref{fig:ADNI_fraction_reg}). Nevertheless, in a practical scenario, domain experts may choose a simpler model at a slight cost in performance. Then, we can measure the tradeoff by varying values of hyperparameters to find an adequate balance~(Figure~\ref{fig:tradeoff_interpre_acc}). The parameter selection is based on the validation set and aligns with the test set results. We see some parameter sensitivity in SUPPORT that supports sharing, but only in a moderate way.

\begin{table}
    \centering
    \begin{tabular}{l|cc|r}
      \multicolumn{4}{l}{ \makecell{Missing features in pattern 4:\\ ABETA, TAU and PTAU at baseline (bl)} } \\
     \bfseries Feature & $\Delta_4$  &  $\theta$ & \bfseries \makecell{$\theta + \Delta_4$}  \\
      \hline
          Age &-0.140 & 0.121 & -0.019 \\
       FDG-PET  & -0.090 & -0.039  & -0.129\\
        \makecell{Whole Brain (bl)}  & 0.000  & -0.045 & -0.044\\
        Fusiform  & 0.016 & 0.021  & 0.037 \\
     ICV &  0.001 & 0.093 & 0.094 \\
     Intercept & -0.10 & 0.18 \\
    \end{tabular}
    \caption{Example of $\Delta_4$ for regression using \spsm{} using ADNI. \spsm{} takes $\gamma=10$ and $\lambda = 13$ as parameters for a single seed. There are 10 missingness pattern in total, while 4 of them have non-zero coefficients for $\Delta$ and pattern-specific intercept. Coefficients are for standardized variables.  \label{tab:Coefficient_change}} 
\end{table}

\section{Related Work}
\label{sec:related}
\emph{Pattern-mixture missingness} refers to distributions well-described by an independent missingness component and a covariate model dependent on this pattern~\citep{rubin1976inference,little1993pattern}. In this work, \emph{pattern missingness} refers to emergent patterns which may or may not depend on observed covariates~\citep{marshall2002prospective}. 
~\citet{fletcher2020missing,lemorvan20a_linear} and \citet{Bertsimas2021PredictionWM} define pattern submodels for flexible handling of test time missingness.
The ExpandedLR method of \citet{lemorvan20a_linear} represents a related method to pattern submodels. However, they neither study coefficient sharing between models nor provide a theoretical analysis of when optimal submodels have partly identical coefficients (sharing, sparsity in specialization). \citet{marshall2002prospective} describes the one-step sweep method using estimated coefficients and an augmented covariance matrix obtained from fully observed and incomplete data at test time. In very recent and so far unpublished work, \citet{Bertsimas2021PredictionWM} present two methods for predicting with test time missingness. First, \emph{Affinely adaptive regression} specializes a shared model by applying a coefficient correction given by a linear function of the missingness pattern. When the number of variables $d$ is smaller than the number of patterns (which could grow as $2^d$), and the outcome is not smooth in changes to missingness mask, this may introduce significant bias. The resulting bias-variance tradeoff differs from our method, and unlike our work, is not justified by theoretical analysis. Second, \emph{Finitely adaptive regression} starts by placing each pattern in the same model, recursively partitioning them into subsets. 

Several deep learning methods which are applicable under test time missingness with or without explicitly attempting to impute missing values have been  proposed~\citep{bengio1995recurrent,che2018recurrent,morvan2020neumiss, lemorvan20a_linear, nazabal2020handling}. The NeuMiss network, discussed briefly in Section~\ref{sec:linear_gaussian}, proposes a new type of non-linearity: the multiplication by the missingness indicator~\citep{morvan2020neumiss}. NeuMiss approximates the specialization term $\Delta_\obs^\top X_\obs$ (along with per-pattern biases) using a deep neural network where both covariates and missingness mask are given as input, sharing parameters across patterns. NeuMiss and Affinely adaptive regression (see above) are similar since their pattern specializations are functions of the inputs and the masks, both in contrast to SPSM. Moreover, neither method attempts to learn sparse specialization terms (e.g., no $\ell_1$ regularization of $\Delta$). 

\section{Conclusion}\label{sec:dis_con}
We have presented sharing pattern submodels (\spsm{}) for prediction with missing values at test time. We enforce parameter sharing through sparsity in pattern coefficient specializations via regularization and analyze \spsm{}'s consistency properties. We have described settings where information sharing is optimal even when the prediction target depends on missing values and the missingness pattern itself. Experimental results using synthetic and real-world data confirm that \spsm{} performs comparably or slightly better than baselines across all data sets without relying on imputation. Notably, the proposed method never performs worse than non-sharing pattern submodels as these do not use the available data efficiently. While \spsm{} is limited to learning linear models, it is not limited to learning from linear systems. An interesting direction is to identify other classes of models developed with interpretability that could benefit from this type of sharing.

%
% ACKNOWLEDGEMENTS 
\section*{Acknowledgements}
We want to thank Devdatt Dubhashi and Marine Le Morvan for their support and fruitful discussions. 

This work was partly supported by WASP (Wallenberg AI, Autonomous Systems and Software Program) funded by the Knut and Alice Wallenberg foundation. 

The computations were enabled by resources provided by the Swedish National Infrastructure for Computing (SNIC) at Chalmers Centre for Computational Science and Engineering (C3SE) partially funded by the Swedish Research Council through grant agreement no. 2018-05973.

Data used in preparation of this article were obtained from the Alzheimer’s Disease Neuroimaging Initiative (ADNI) database (adni.loni.usc.edu). As such, the investigators within the ADNI contributed to the design and implementation of ADNI and/or provided data but did not participate in the analysis or writing of this report. A complete listing of ADNI investigators can be found at: \url{http://adni.loni.usc.edu/wp-content/uploads/how_to_apply/ADNI_Acknowledgement_List.pdf}
%
% REFERENCES
%
\fontsize{9.1pt}{10.2pt}
\selectfont
\bibliography{references}

%\end{document}

\clearpage
%
% APPENDIX
\appendix

\section{Technical appendix}

\subsection{Variable dependencies}
The assumed (causal) dependencies of the variables $X,M,Y$ are represented in a directed graph in Figure~\ref{fig:DAG}.

\begin{figure}[t]
    \centering
    \begin{tikzpicture}[node distance={8mm}, thick, main/.style = {draw, circle}, unobs/.style = {draw, dashed, circle}, every arrow/.append style={dash, thick}] 
        \coordinate (origin);
        \node[main] (1) {X}; 
        \node[main] (2) [right of=1, xshift=2.5cm] {$\widetilde{X}$}; 
        \node[main] (4) [below of=2, yshift=-.7cm] {M};
        \node[main] (5) [below of=1, yshift=-.7cm] {Y};
        \node[unobs] (6) [right of=1, yshift=0.7cm, xshift=.85cm] {U}; 

        \draw[->](1) -- (2); 
        \draw[->](4) -- (5); 
        \draw[->](4) -- (2);
        \draw[->](1) -- (5);
        \draw[->](1) -- (4);
        \draw[->](6) -- (1);
        \draw[->](6) -- (4);
        \draw[->](6) -- (5);
    \end{tikzpicture}
    \caption{Directed graph showing assumed probabilistic dependencies. $\tilde{X}$ is a deterministic function of $X, M$. Unobserved variables $U$ may influence both covariates $X$, missingness $M$ and the outcome $Y$, ruling out  `missing at random' (MAR).}\label{fig:DAG}
\end{figure}
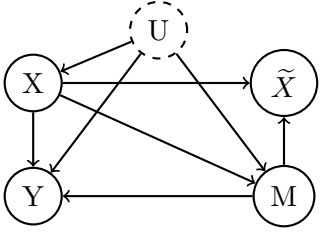

\subsection{Consistency in the general case}
\label{app:consistency}
\begin{thmprop}\label{prop:consistency}
For each pattern $m$, the minimizers $(\theta^*, \Delta_\obs^*)$ of \eqref{eq:spsm_ols} are consistent estimators of the best linear approximation to $\E[Y \mid X_\obs, M=m]$, 
$$
\lim_{n \rightarrow \infty} (\theta_\obs^* + \Delta_\obs^*) = \min_{\eta} \E[(\eta^\top X_\obs - Y)^2 \mid M=m]~.
$$
When the true outcome is linear, $Y = \eta_\obs^\top X_\obs + \epsilon$ with Gaussian errors $\epsilon$, $\lim_{n \rightarrow \infty} (\theta_\obs^* + \Delta_\obs^*) = \eta_\obs$~.
\end{thmprop}
\begin{proofsketch}
Minimizers $\Delta^*$ and $\theta^*$ will have bounded norm due to the quadratic form of the objectives. This, in the limit $n \rightarrow \infty$, regularization terms vanish due to normalization with $n$ and the minimizers $(\theta^*, \{ \Delta_\obs^*\})$ are invariant to additive transformations; with $c \in \bbR^{d_m}$, $\theta'_\obs = \theta^*_\obs + c$ and $\Delta'_\obs = \Delta^*_\obs - c$ also minimize the objective. Choosing $c = -\theta^*_\obs$, we get $\theta'_\obs = 0$ and the objective becomes separable in $m$. As a result, the objective can be written as $k$ standard least squares problems, one for each pattern. As is well known, for additive sub-Gaussian noise, the minimizers of these problems are consistent for the best linear approximation to the corresponding conditional mean.
\end{proofsketch}

\subsection{Proof of Proposition~\ref{prop:main}}
\label{app:thm}
\begin{thmprop*}[Proposition~\ref{prop:main} Restated]
Suppose covariates $X$ and outcome $Y$ obey Condition~\ref{cond:lineargauss} (are linear-Gaussian). Then, the Bayes-optimal predictor for an arbitrary missingness mask $m \in \cM$, is
\begin{equation*}
    f^*_m = \mathbb{E}[Y \mid X_\obs, m] = (\theta_\obs + \Delta_\obs)^\top X_\obs + C_m
\end{equation*}
where $C_m \in \mathbb{R}$ is constant with respect to $X_\obs$ and
\begin{equation*}
    \Delta_\obs = ({\Sigma^{-1}_{\obs, \obs}}) \Sigma_{\obs, \miss}  \theta_\miss~.%
\end{equation*}%
\end{thmprop*}
\begin{proof}
By properties of the multivariate Normal distribution, we have that 
\begin{align*}
    & \mathbb{E}_{X_{\miss}}[ X_{\miss} \mid X_\obs] \\ 
    & = \mathbb{E}[X_{m}]  + \Sigma_{m, \obs} {\Sigma^{-1}_{\obs, \obs}} ( X_\obs - \mathbb{E}[X_\obs] )
\end{align*}
and as a result, following the reasoning above, 
\begin{align*}
    & \mathbb{E}[Y \mid X_\obs]\\
    &  = (\theta_\obs + (\Sigma_{\miss, \obs} {\Sigma^{-1}_{\obs, \obs}}) \theta_{m})^\top X_\obs + C_m \\
    & = (\theta_\obs + \Delta_m)^\top X_\obs + C_m, 
\end{align*}
    where $C_m = \theta_m^\top (\mathbb{E}[X_{m}] - \Sigma_{\miss, \obs} {\Sigma^{-1}_{\obs, \obs}}\E[X_\obs]) + \alpha_m $, 
    which is constant w.r.t. $X_\obs$.
\end{proof}

\subsection{Sparsity in optimal model}
\begin{thmprop*}[Proposition~\ref{prop:sparsity} restated]
Suppose that a covariate $j\in [d]$ is observed under pattern $m$, i.e., $m_j = 0$, and assume that $X_j$ is non-interactive with every covariate $X_{j'}$ that is missing under $m$. Then $(\Delta_\obs)_j = 0$. %
\end{thmprop*}%
\begin{proof}
Let $\bartheta_\obs = \theta_\obs + \Delta_\obs$. Recall that $S = \Sigma^{-1}$ is the precision matrix for $X \sim \mathcal{N}(\boldsymbol{\mu}, \Sigma)$ and permute the rows and columns of $\Sigma$ into observed and unobserved parts, such that, without loss of generality, we can write 
\begin{equation}
    \Sigma=\left[\begin{array}{cc}
         \Sigma_{\obs, \obs} & \Sigma_{\obs,\miss} \\
         \Sigma^T_{\obs,\miss} & \Sigma_{\miss,\miss} 
    \end{array}\right]~.
\end{equation}
Note that by the definition of $\Delta_\obs$ (Proposition~\ref{prop:main}),
\begin{align*}
    \Sigma\left[\begin{array}{c}
         \Delta_\obs  \\
          -\bartheta_\obs
    \end{array}\right] 
    & =\left[\begin{array}{cc}
         \Sigma_{\obs,\obs} & \Sigma_{\obs,\miss} \\
         \Sigma^T_{\obs,\miss} & \Sigma_{\miss,\miss} 
    \end{array}\right]\left[\begin{array}{c}
         \Delta_\obs  \\
          -\bartheta_\obs
    \end{array}\right] \\ 
    & =\left[\begin{array}{c}
          0 \\
         g_m 
    \end{array}\right], 
\end{align*}
where $g_m$ is a suitable vector. Hence
\begin{equation*}
    \left[\begin{array}{c}
         \Delta_\obs  \\
          -\bartheta_\obs
    \end{array}\right]=\Sigma^{-1}\left[\begin{array}{c}
          0 \\
         g_m 
    \end{array}\right]=S\left[\begin{array}{c}
          0 \\
         g_m 
    \end{array}\right]
\end{equation*}
We conclude the result by noting that $(\Delta_\obs)_j$ is zero if in the $j$th row of $S$, all entries corresponding to the unobserved part is zero.
\end{proof}

\subsection{Consistency in linear-Gaussian DGPs}
\label{app:consistency_lin}

\begin{thmthm*}[Theorem~\ref{thm:consistency} restated]
Suppose that Condition~\ref{cond:lineargauss} holds with parameters $(\theta, \{\Delta_\obs\})$ as in Proposition~\ref{prop:main}, such that, for each covariate $j$, the number of patterns $m$ for which $m_j = 0$ and $(\Delta_\obs)_j = 0$ is strictly larger than the number of patterns $m'$ for which $m'_j=0$ and $(\Delta_{\obs'})_j \neq 0$. 
Then, with $\gamma=0$ and $\lambda >0$, the true parameters $(\theta, \{\Delta_\obs\})$ are  the unique  solution to \eqref{eq:spsm_ols} in the large-sample limit,  $n \rightarrow \infty$. 
\end{thmthm*}
\begin{proof}
Consider the optimization problem in eq. \eqref{eq:spsm_ols} with $\gamma=0$.
In the large-sample limit $(n \rightarrow \infty)$, minimizers of the empirical risk over $n$ samples will also minimize the expected risk and, since the outcome is linear-Gaussian, satisfy the constraint in eq. \eqref{eq:L1_decomposition}. 
Then, solving \eqref{eq:spsm_ols} is equivalent to solving the following problem:
\begin{align}
& \underset{\theta', \{\Delta'_\obs\}}{\text{minimize}}
& & \suml_{m}\|\Delta'_\obs\|_1 \label{eq:L1_decomposition} \\
& \text{subject to}
& & \theta_\obs'+\Delta_\obs'=\theta_\obs + \Delta_\obs,\quad m \in \mathcal{M} \nonumber
\end{align}
Many parameters $(\theta', \Delta')$ can satisfy the constraint, due to translational invariance. However, for any value of $\lambda >0$, regularization in \eqref{eq:spsm_ols} steers the solution towards the one with the smallest norm, $\|\Delta'\|_1$.
The reasoning is similar to the  argument in the proof of  Proposition~\ref{prop:consistency}, adding the assumption that the true system is linear-Gaussian.  Under the added assumptions of Theorem~\ref{thm:consistency}, we can now prove that we also get the correct decomposition.

Take a solution  $\theta^*, \{\Delta^*_\obs\}$ of \eqref{eq:L1_decomposition}. For simplicity of notation below, let vectors $\theta, \Delta_\obs, \theta^*, \Delta^*_\obs$ always be indexed such that the same index $j$ refers to coefficients corresponding to the same covariate $X_j$. Next, define $I_j=\{m \mid m_j=0,\ (\Delta_\obs)_j=0\}$ to be the set of patterns where covariate $j$ is observed and without specialization under the optimal model. Similarly, define $I^c_j=\{m \mid m_j=0,\ (\Delta_\obs)_j\neq 0\}$ to be the set of patterns where covariate $j$ is observed and needs specialization. First, note that 
\begin{align*}
    \suml_m\|\Delta^*_\obs\|_1 & =\suml_j\suml_{m \mid m_j=0}\left|(\Delta^*_\obs)_j\right| \\
    & =\suml_j\suml_{m\in I_j}\left|(\Delta^*_\obs)_j\right|+\suml_j\suml_{m\in I^c_j}\left|(\Delta^*_\obs)_j\right|.
\end{align*}
For $m\in I_j$, we have $\theta^*_j+(\Delta^*_\obs)_j=\theta_j$. Hence
\begin{equation}
    \suml_j\suml_{m \in I_j}\left|(\Delta^*_\obs)_j\right|=\suml_j|I_j||\theta_j-\theta_j^*|
\end{equation}
For $m\in I^c_j$, we have $\theta^*_j+(\Delta^*_\obs)_j=\theta_j+(\Delta_\obs)_j$ and hence by the triangle inequality, we have
\begin{eqnarray}
    &\suml_j\suml_{k\in I^c_j}\left|(\Delta^*_k)_j\right|\geq\suml_j\suml_{k\in I^c_j}\left(\left|(\Delta_k)_j\right|-|\theta_j-\theta_j^*|\right)=\nwl
    &\suml_m\|\Delta_\obs\|_1-\suml_j|I_j^c||\theta_j-\theta_j^*|
\end{eqnarray}
We conclude that
\begin{align*}
    \suml_m\|\Delta^*_\obs\|_1 & \geq \suml_m\|\Delta_\obs\|_1+\suml_j(|I_j|-|I_j^c|)|\theta_j-\theta_j^*| \\
    & \geq \suml_m\|\Delta_\obs\|_1
\end{align*}
where the last inequality is by the assumption. This provides the desired result.
\end{proof}

\section{Experiment details}
\subsection{Real world data sets}
\paragraph{ADNI}
The compiled data set includes 1337 subjects that were preprocessed by one-hot encoding of the categorical features and standardized for the numeric features. The processed data has 37 features and 20 unique missingness patterns. The label set is quite unbalance showing 1089 patients who do not change from their baseline diagnosis, and 248 do. The regression task targets predicting the result of the cognitive test ADAS13 (Alzheimer's Disease Assessment Scale) at a 2 year follow-up~\citep{mofrad2021cognitive} based on available data at baseline.

\paragraph{SUPPORT}
The data set contains 9104 subjects represented by 23 unique missingness pattern. 
The following 10 covariates were selected and standardized: partial pressure of oxygen in the arterial blood (pafi), mean blood pressure, white blood count, albumin, APACHE III
respiration score, temperature, heart rate per minute, bilirubin, creatinine, and sodium.

\subsection{Details of the baseline methods}\label{app:B}
We compare to the following baseline methods: 
\begin{itemize}
    \item[]  \textbf{Imputation + Ridge / logistic regression (\ridge/\lr)} the data is first imputed (see below) and a ridge or logistic regression is fit on the imputed data. The implementation in SciKit-Learn was used~\citep{Pedregosa}.
    The ridge coefficients are shirked by imposing a penalty on their size. They are a reduced factor of the simple linear regression coefficients and thus never attain zero values but very small values~\citep{tibshirani1996regression}
    \item[] \textbf{Imputation + Multilayer perceptron (\mlp)}: The MLP estimator is based on a single hidden layer of size $\in [10, 20, 30]$ followed by a ReLu activation function and a softmax layer for classification tasks and a linear layer for regressions tasks. As input, the imputed data is concatenated with the missingness mask. The MLP is trained using ADAM~\citep{kingma2014adam}, and the learning rate is initialized to constant ($0.001$) or adaptive. We use the implementation in SciKit-Learn~\citep{scikit-learn}.
    \item[] \textbf{Pattern submodel (\psm)}: For each pattern of missing data, a linear or logistic regression model is fitted, separately regularized with a $\ell_2$ penalty. Following \citet{fletcher2020missing}, for patterns with fewer than $2*d$ samples available, a complete-case model (CC) is used.
    %also to combat vairance
    Our implementation of \psm{} is based on a special case of our \spsm{} implementation where regularization is applied over all patterns and not in each pattern separately. To enforce fitting separated submodels for each pattern, we set $\gamma = 1e^8$ and $\lambda = 0$.
    \item[] \textbf{XGBoost (\xgb)}: XGBoost is an implementation of gradient boosted decision trees. Note, XGBoost supports missing values by default~\citep{chen2019package}, where branch directions for missing values are learned during training. A logistic classifier is then fit using XGBClassifier while regression tasks are trained with the XGBRegressor~\citep{scikit-learn}. We set the hyperparameters to 100 for the number of estimators used, and fix the learning rate to 1.0. The maximal depth of the trees is $\in [5,10,15]$.
\end{itemize}

Imputation methods and hyperparameters for all methods were selected based on the validation portion of random 64/16/20 training/validation/test splits. 
Results were averaged over five random splits of the data set.   
The performance metrics for classification tasks were accuracy and the Area Under the ROC Curve (AUC). For regression tasks, we use the mean squared error (MSE) and 
%the general coefficient of determination, 
the R-square, ($R^2$) value, representing the proportion of the variance for a dependent variable that's explained by an independent variable, taking values in [$- \infty$, 1] where negative values represent predictions worse than the mean \citep{dancer2005r}. Confidence intervals at significance level $\alpha=0.05$ are computed based on the number of test set samples. 
For accuracy, MSE and $R^2$ we use a Binomial proportion confidence interval~\citep{fagerland2015recommended} and for AUC we use the classical model of~\citep{HanleyMcNeil}.

\paragraph{Computing Infrastructure}
The computations required resources of 4 compute nodes using two Intel Xeon Gold 6130 CPUS with 32 CPU cores and 384 GiB memory (RAM). Moreover, a local disk with the type and size of SSSD 240GB with a local disk, usable area for jobs including 210 GiB was used. 
Inital experiments are run on a Macbook using macOS Montery with a  2,6 GHz 6-Core Intel Core i7 processor.

\section{Additional experimental results}
\subsection{Simulation results}~\label{app:simu_resul}
Results for synthetic data with missingness Setting B (pattern-dependent) and Setting C (MCAR) can be found in Figures~\ref{fig:synth_exp_B} and~\ref{fig:synth_exp_C}, respectively.
\begin{figure}[t]
\centering
\includegraphics[width=0.8\columnwidth]{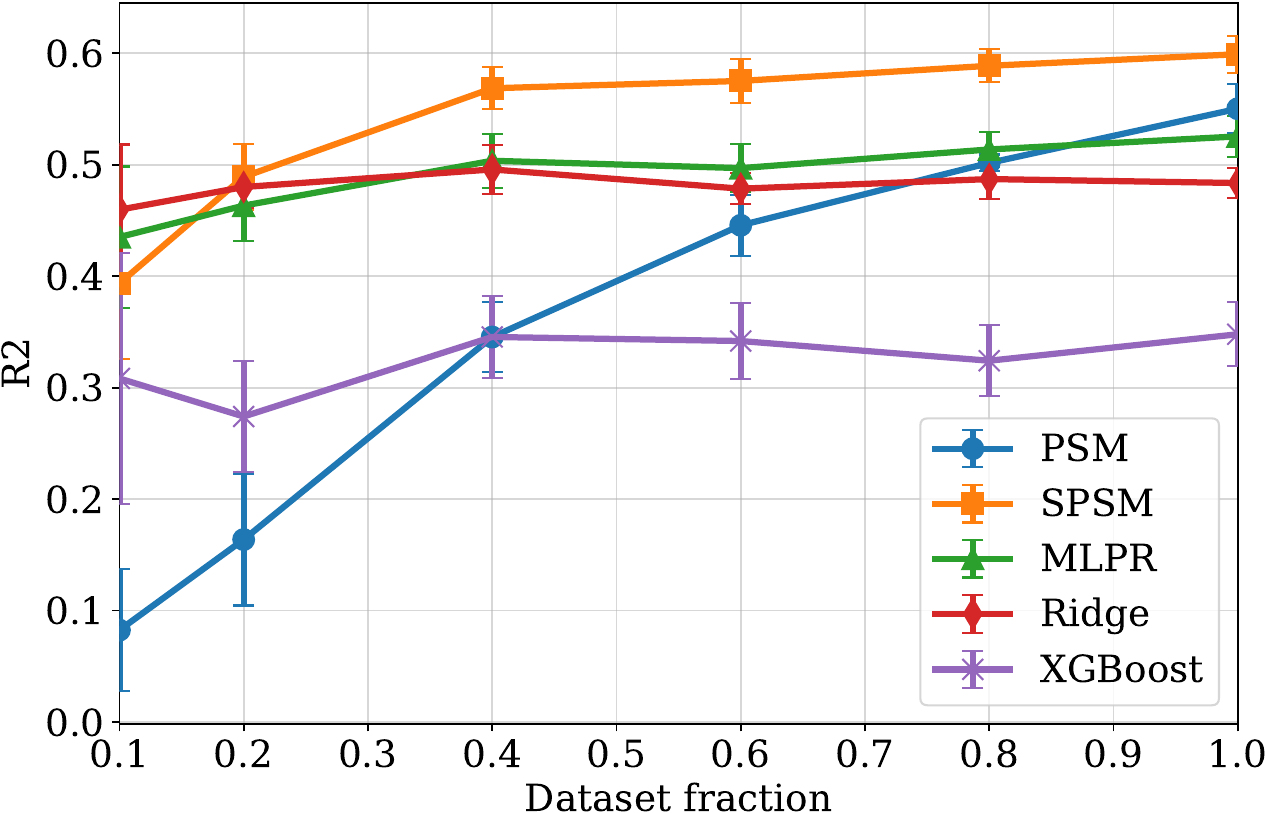} 
\caption{Performance on simulated data Setting B. Error bars indicate standard deviation over 5 random data splits. The complete data set has $n=10000$ samples.}
\label{fig:synth_exp_B}
\end{figure}

\begin{figure}[t]
\centering
\includegraphics[width=0.8\columnwidth]{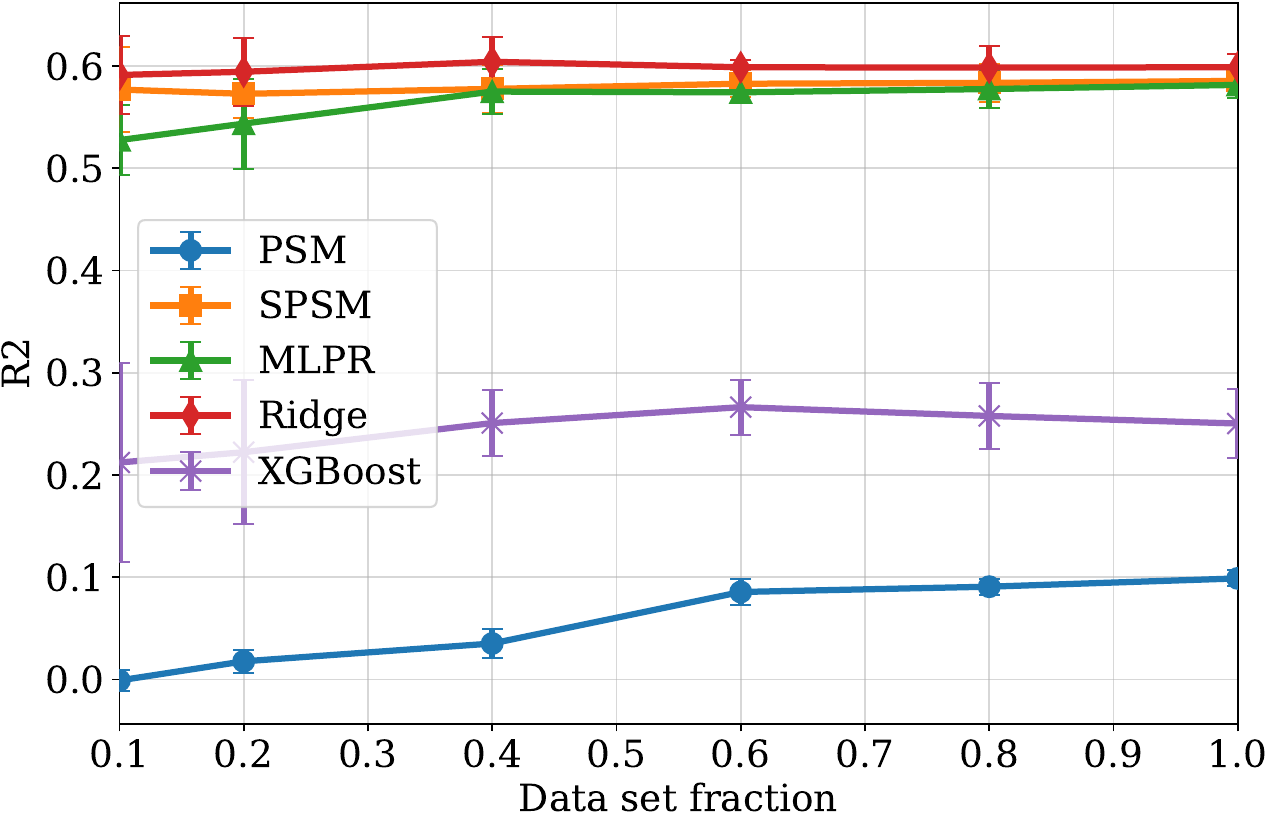} 
\caption{Performance on simulated data Setting C (MCAR). Error bars indicate standard deviation over 5 random data splits. The complete data set has $n=10000$ samples.}
\label{fig:synth_exp_C}
\end{figure}

\subsection{Results for ADNI and SUPPORT}
A figure illustrating the performance on SUPPORT with varying data set size is given in Figure~\ref{fig:SUPPORT_fraction_reg}.
Table~\ref{tab:appendix_data_results_linear_MSE} presents the MSE score as an additional performance metric for the regression tasks using ADNI and SUPPORT data. 
For the MAR setting in the SUPPORT data, we present the results for classification and regression tasks in Table~\ref{tab:appendix_non_Mnar_classifier} and Table~\ref{tab:appendix_data_results_regession_non_mnar}. Moreover, the full table of pattern 4 non-zero coefficients with the corresponding missing features is displayed in Table~\ref{tab:Full_table_all}. 

\begin{figure}[t]
    \centering
    \includegraphics[width=.8\columnwidth]{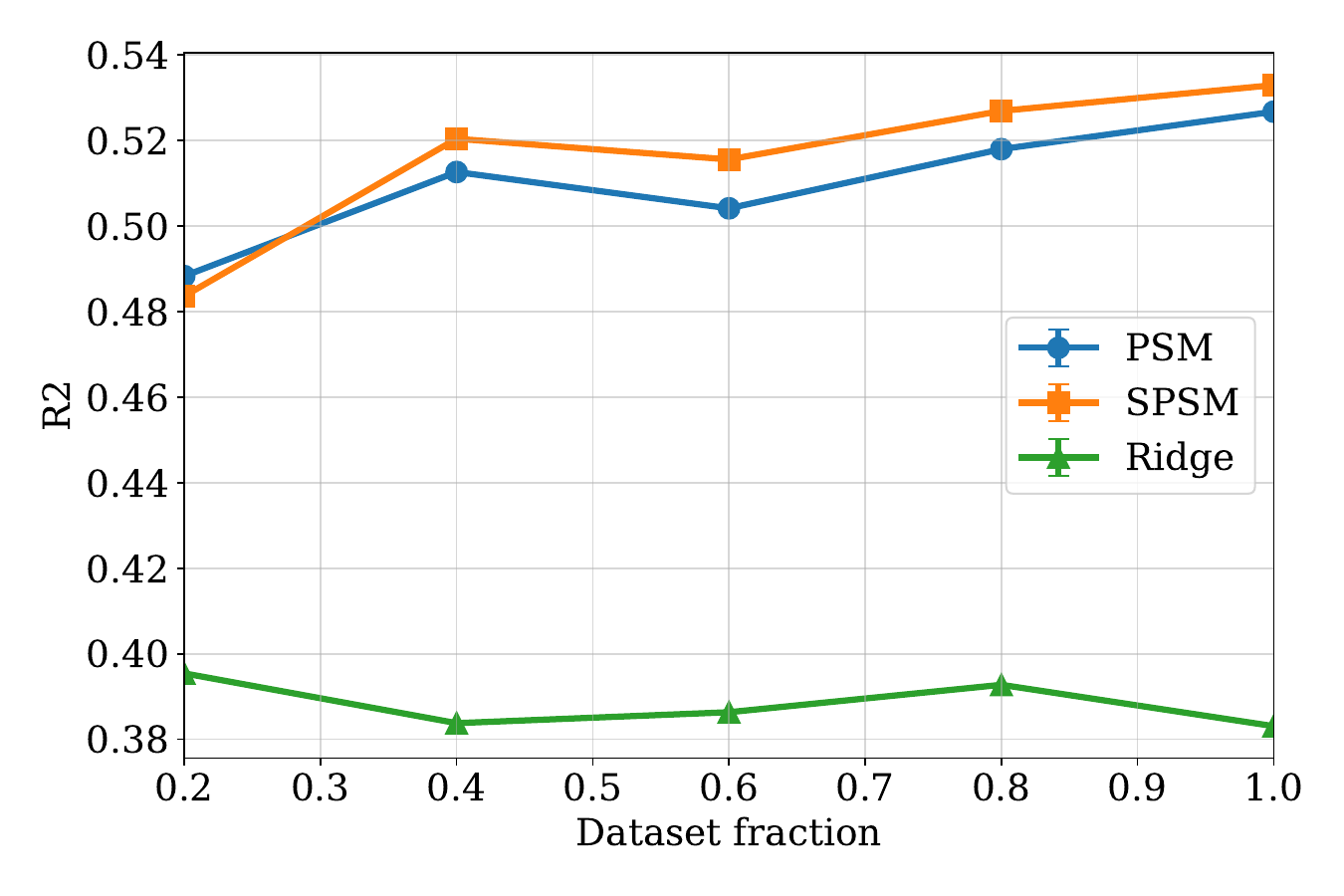}
    \caption{Performance on SUPPORT data for a regression task. Error bars indicate standard deviation over 5 random subsamples of the data.}
    \label{fig:SUPPORT_fraction_reg}
\end{figure}

 \begin{table}[t!]
    \centering
    \begin{tabular}{rlll}
      & \bfseries \makecell{Linear \\ Methods} & \bfseries MSE \\
      \hline
        \multicolumn{3}{l}{\emph{ADNI}} \\
        \hline
        & \ridge{}, $I_0$ & 0.36  (0.26, 0.46) \\
        & \xgb{}, $I_\mu$  &0.60 (0.48, 0.74)  \\
        & \mlp{}, $I_\mu$ & 0.37 (0.27, 0.47)  \\
        &\psm{} & 0.50 (0.38, 0.62)  \\
        & \spsm{}  & 0.35 (0.25, 0.45) \\
      \hline
        \multicolumn{3}{l}{\emph{SUPPORT}} \\
        \hline
        &\ridge{},$I_0$  & 0.61 (0.56, 0.66) \\
        & \xgb{}, $I_\mu$   & 0.69 (0.63, 0.75)\\
        & \mlp{}, $I_0$  & 0.44 (0.39, 0.48) \\
        & \psm{}  & 0.47 (0.42, 0.52) \\
        & \spsm{}  & 0.47 (0.42, 0.51) \\
    \end{tabular}
    \caption{Experimental results of regression methods for ADNI and SUPPORT data set.}\label{tab:appendix_data_results_linear_MSE}
\end{table}

\begin{table}[t]
    \centering
    \begin{small}
    \begin{tabular}{rlcc}
      \multicolumn{2}{l}{\bf Regressions} & $\boldsymbol{R}^2$ & \bfseries MSE\\
        \hline
        \multicolumn{3}{l}{SUPPORT} \\
        \hline
        & \ridge{},  $I_0$ & 0.38 (0.34, 0.41)& 0.62 (0.57, 0.67)\\
        &\xgb{}, $I_\mu$  &  0.27 (0.23, 0.31) & 0.73 (0.67, 0.78) \\
        & \mlp{},  $I_0$ & 0.55 (0.52, 0.58) & 0.45 (0.40, 0.49)\\
        & \psm{} & 0.51 (0.48, 0.54) & 0.49 (0.44, 0.53)\\
        & \spsm{} & 0.52 (0.49, 0.58) & 0.47 (0.42, 0.51)\\
    \end{tabular}
    \end{small}
    \caption{Experimental results of regression methods for SUPPORT data set MAR. }\label{tab:appendix_data_results_regession_non_mnar}
\end{table}

\begin{table}
    \centering
    \begin{small}
    \begin{tabular}{rlcc}
      \multicolumn{2}{l}{\bfseries Classifiers}  & \bfseries AUC & \bfseries Accuracy \\
        \hline
        \multicolumn{3}{l}{SUPPORT} \\
        \hline
        & \lr{}, $I_0$ &  0.82 (0.80, 0.84) &  0.75 (0.72, 0.78) \\
        & \xgb{}, $I_0$ & 0.83 (0.81, 0.85)  & 0.76 (0.74, 0.78) \\
        & \mlp{}, $I_0$ & 0.85 (0.84, 0.87) & 0.78 (0.76, 0.81)\\
        & \psm{} & 0.83 (0.81, 0.85) & 0.78 (0.74, 0.80) \\
        & \spsm{}  & 0.83 (0.81, 0.85)  & 0.76 (0.73, 0.80)\\
    \end{tabular}
    \end{small}
    \caption{Experimental results of classifiers for SUPPORT data with MAR.}\label{tab:appendix_non_Mnar_classifier}
\end{table}

\begin{table}[t]
    \centering
    \begin{small}
    \begin{tabular}{l|ll|l}
    \multicolumn{3}{l}{ \makecell{Missing features in pattern 0:\\ None} } \\
    \bfseries Feature & \bfseries \makecell{$\Delta_m$}  &  \bfseries \makecell{$\theta$} & \bfseries $\theta$ + $\Delta_m$ \\ 
      \hline
        Age &-0.038 & 0.121 & 0.082 \\
 EDUCAT    &    0.014 & -0.005 & 0.009 \\
  APOE4      &       0.046  &  -0.010  &  0.035 \\
   FDG        &       -0.032  &  -0.039  &  -0.071 \\
  ABETA    &     0.027  & -0.000  &  0.027\\
 LDELTOTAL     &    0.051  &  -0.391  &  -0.340 \\
0  Entorhinal     &   0.007  &  -0.131  &  -0.124\\
 ICV           &    0.013  &  0.093  &  0.106\\
 Diagnose MCI    &    0.078  & -0.139  &  -0.061\\
 GEN Female  & -0.054& 0.003  &  -0.050\\
  GEN Male  &   0.000  &  0.062  &  0.062\\
  \makecell{Not Hisp/ \\Latino} &  0.047  &  -0.114  &  -0.067\\
  Married &  0.115  &  -0.159  &  -0.044\\
        \hline
    \multicolumn{3}{l}{ \makecell{Missing features in pattern 1:\\ FDG} } \\
      \hline
       Age &-0.052 & 0.121 & 0.069 \\
        \hline
      \multicolumn{3}{l}{ \makecell{Missing features in pattern 4:\\ ABETA, TAU and PTAU at baseline (bl)} } \\
      \hline
          Age &-0.140 & 0.121 & -0.019 \\
       FDG  & -0.090 & -0.039  & -0.129\\
        Whole Brain  & 0.000  & -0.045 & -0.044\\
        Fusiform  & 0.016 & 0.021  & 0.037 \\
         ICV &  0.001 & 0.093 & 0.094 \\
    \multicolumn{3}{l}{ \makecell{Missing features in pattern 10:\\ FDG, ABETA (bl), TAU (bl), PTAU (bl)} } \\
      \hline
        APOE4  & 0.038 &-0.010 &0.027\\
    \end{tabular}
    \end{small}
    \caption{Full table showing $\Delta_m$ in the regression task using SPSM for ADNI.}\label{tab:Full_table_all}
\end{table}

\begin{table}[t]
    \centering
    \begin{tabular}{rcc}
     {\bfseries \makecell{Pattern\\number}} & \makecell{Number of\\ subjects\\ per pattern} &
     $\boldsymbol{R}^2$\\
      \hline
      0 & 119& 0.64 (0.53, 0.75)\\
      1 & 30 & 0.30 (-0.10, 0.55)\\
      6 & 27 & 0.71 (0.50, 0.92) \\
      10 & 28 & 0.71 (0.50, 0.91)\\
     others & $\leq$ 13 & undefined or insignificant\\
    \end{tabular}
    \caption{A minimum sample size is required for \spsm{} to maintain predictive performance}\label{tab:appendix_sample_sizes}
\end{table}

\begin{figure}[t]
    \centering
    \includegraphics[width=1\columnwidth]{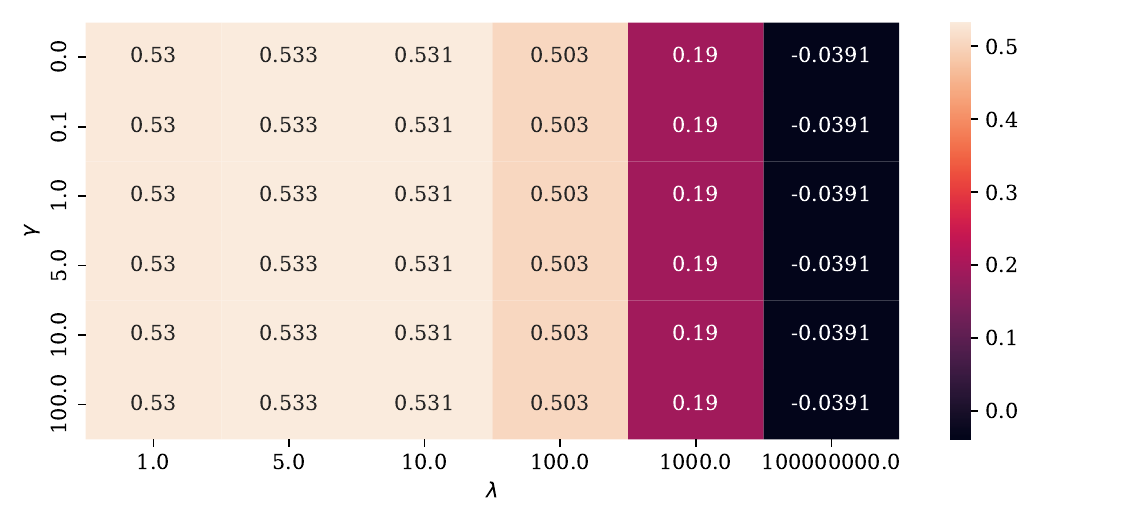}
    \caption{Heatmap visualizing the tradeoff between interpretability and prediction power including different hyperparameter values for $\gamma$ and $\lambda$, expressed by the $R^2$ using SUPPORT data. Each cell is indicating a  $\gamma$,$\lambda$ combination, e.g. 1,100 represents 1 = $\gamma$ and 100 = $\lambda$. }
    \label{fig:tradeoff_interpre_acc}
\end{figure}

\subsection{HOUSING data}\label{appendix:housing}
The Ames Housing data set (HOUSING)~\cite{de2011ames} was compiled by Dean De Cock for use in data science education. The data set describes the sale of individual residential property in Ames, Iowa from 2006 to 2010. The data set contains 2930 observations and a large number of explanatory variables (23 nominal, 23 ordinal, 14 discrete, and 20 continuous) involved in assessing home values. In this study we used a subset of the features 27 features to describe the main characteristics of a house. Examples of features included are measurements about the land ('LotFrontage', 'LotArea', 'LotShape', 'LandContour', 'LandSlope'), the 'Neighborhood', and 'HouseStyle', when the house was build ('YearBuilt'), or remodeled ('YearRemodAdd'). Moreover, features describing the outside of the house ('RoofStyle', 'Foundation'), technical equipment  ('Heating', 'CentralAir', 'Electrical', 'KitchenAbvGr', 'Functional', 'Fireplaces', 'GarageType', 'GarageCars', 'PoolArea',
'Fence', 'MiscFeature'), and information about previous house selling prices and conditions  ('MoSold', 'YrSold', 'SaleType',  'SaleCondition'). 
The numeric features where standardized and the categorical ones are one-hot-encoded during preprocessing. The HOUSING data set shows 15 different missingness patterns.
An exploratory analysis has shown that the house sale prices are somehow skewed, which means that there is a large amount of asymmetry. The mean of the characteristics is greater than the median, showing that most houses were sold for less than the average price. 
In the classification predictions, we look if the sale prices for a house are above or below the median, while for regression tasks we predict the sale price for a house.

We report the results of the HOUSING data set in Table~\ref{tab:appendix_housing}. In classification, on average a high performance over all models, whereas the best performing one, \xgb{} achieves an AUC of 0.96 and an accuracy of 0.91. \spsm{} achieves only slightly lower prediction power of 0.95 AUC and 0.88 accuracies than \xgb{}. While \lr, \xgb{} and \mlp{} depend on mean or zero imputation, \psm and \spsm{} are able to achieve comparable results without adding bias to their prediction with high confidence on average.
For the HOUSING regression, the validation power suggested $\gamma = 10$, $\lambda = 100$ for \spsm{}, resulting in an $R^2$ of 0.64 and an MSE of 0.39. This result is better than for \psm ($R^2$ of 0.58 and MSE of 0.46) and thus demonstrates the benefit of coefficient sharing in \spsm{} compared to no sharing. Although \ridge{}, and \mlp{} perform better the differences are only marginal to \spsm{}. The best performing model is the black-box method of \xgb{} achieving an $R^2$ of 0.76 and MSE of 0.27 indicating the non-linearity of the data set.

\begin{table}[t]
    \centering
    \begin{small}
    \begin{tabular}{rlcc}
      \multicolumn{2}{l}{\bfseries Housing} \\
      \hline
        \multicolumn{2}{c}{\emph{Classification}} &\bfseries AUC & \bfseries Accuracy \\
        \hline
        & \lr{},  $I_\mu$ & 0.96 (0.94, 0.98) & 0.90 (0.85, 0.95) \\
        &\xgb{}, $I_0$ & 0.96 (0.94, 0,98) &  0.91 (0.87, 0.96)  \\
        & \mlp{},  $I_0$ & 0.96 (0.93, 0.98) & 0.90 (0.85, 0.94)\\
        & \psm{} & 0.93 (0.90, 0.96) & 0.88 (0.83, 0.93)\\
        & \spsm{} & 0.95 (0.92, 0.97)&0.88 (0.83, 0.94)\\
      \hline
        \multicolumn{2}{c}{\emph{Regression}}& $\boldsymbol{R}^2$ & \bfseries MSE\\
        \hline
        & \ridge{},  $I_\mu$ & 0.68 (0.62, 0.75) & 0.35 (0.25, 0.44) \\
        &\xgb{}, $I_0$ & 0.76 (0.70, 0.81)  & 0.27 (0.18,0.35)  \\
        & \mlp{},  $I_0$&  0.64 (0.58, 0.71) & 0.39 (0.29, 0.49) \\
        & \psm{} & 0.58 (0.50, 0.65) & 0.46 (0.35, 0.57) \\
        & \spsm{} &0.64 (0.57, 0.71) & 0.39 (0.29, 0.49)\\
    \end{tabular}
    \end{small}
    \caption{Experimental results of classification and regression methods for HOUSING data set. }
    \label{tab:appendix_housing}
\end{table}

\subsection{Analysis of interpretability}\label{app:interpre}
By enforcing sparsity in pattern specialization, we ensure that the resulting subset of features is reduced to relevant differences which will foster interpretability for domain experts; \spsm{} allows for more straight-forward reasoning about the similarity between submodels and the effects of missingness. \citet{lipton2016modeling} provides qualitative design criteria to address model properties and techniques thought to confer interpretability. We will show that \spsm{} satisfies some form of transparency by asking, i.e., \textit{how does the model work?}.
As stated in \citep{lipton2016modeling}, transparency is the absence of opacity or black-boxness meaning that the mechanism by which the model works is understood by a human in some way. We evaluate transparency at the level of the entire model (simulatability), at the level of the individual components (e.g., parameters) (decomposability), and at the level of the training algorithm (algorithmic transparency).
First, simulatability refers to contemplating the entire model at once and is satisfied in \spsm{} by it's nature of a sparse linear model, as produced by lasso regression~\cite{tibshirani1996regression}. Moreover, we claim that \spsm{} is small and simple~\citep{rudin2019stop}, in that we allow a human to take the input data along with the parameters of the model and perform in a reasonable amount of time all the computations necessary to make a prediction in order to fully understand a model. 
The aspect of decomposabilty~\citep{lipton2016modeling} can be satisfied by using tabular data where features are intuitively meaningful. To that end, we use two real-world tabular data sets in the experiments and present the coefficient values for input features in Table~\ref{tab:Coefficient_change}. Moreover, one can choose to display the coefficients in a standardized or non-standardized way to provide even better insights. The comprehension of the coefficients depends also on domain knowledge. 
Finally, algorithmic transparency is given in \spsm{}, since in linear models, we understand the shape of the error surface and have some confidence that training will converge to a unique solution, even for previously unseen test data.
Additionally,~\citet{henelius2017interpreting} claims that knowing interactions between two or more attributes makes a model more interpretable. \spsm{} shows in $\theta + \Delta$ the coefficient specialization between the main model and the pattern-specific model and therefore reveals associations between attributes.
\end{document}